\documentclass[11pt]{article}
\usepackage[margin=1.1in]{geometry}
\usepackage[utf8]{inputenc}
\usepackage{amsmath}
\usepackage{amssymb}
\usepackage{amsthm}
\usepackage{hyperref}
\usepackage{bm}
\usepackage{tikz}
\usepackage{graphicx}
\usepackage[dvipsnames]{xcolor}
\usepackage{subcaption}
\usepackage{url}
\usepackage{asymptote}
\usepackage{mathtools}
\usepackage{pgfplots}
\pgfplotsset{compat=1.18}
\usepackage{cleveref}
\usepackage{thmtools}

\usepackage[font=small,labelfont=bf]{caption}

\newtheorem{theorem}{Theorem}
\newtheorem{corollary}[theorem]{Corollary}
\newtheorem{proposition}[theorem]{Proposition}
\newtheorem{lemma}[theorem]{Lemma}

\newtheorem{definition}[theorem]{Definition}

\newtheorem{remark}[theorem]{Remark}

\newcommand{\RR}{\mathbb{R}}

\newcommand{\norm}[1]{\left\|#1\right\|}

\DeclareMathOperator{\rank}{rank}
\DeclareMathOperator{\argmin}{argmin}
\DeclareMathOperator{\ReLU}{ReLU}

\usepackage[round]{natbib}
\hypersetup{
    colorlinks,
    linkcolor={red!50!black},
    citecolor={blue!50!black},
    urlcolor={blue!80!black}
}

\title{How Many Features Can a Language Model Store\\Under the Linear Representation Hypothesis?}
\author{Nikhil Garg \quad Jon Kleinberg \quad Kenny Peng}
\date{Cornell University}

\begin{document}

\maketitle

\begin{abstract}
    We introduce a mathematical framework for the linear representation hypothesis (LRH), which asserts that intermediate layers of language models store features linearly. We separate the hypothesis into two claims: linear \textit{representation} (features are linearly embedded in neuron activations) and linear \textit{accessibility} (features can be linearly decoded). We then ask: How many neurons $d$ suffice to both linearly represent and linearly access $m$ features? Classical results in compressed sensing imply that for $k$-sparse inputs, $d = O(k\log (m/k))$ suffices if we allow non-linear decoding algorithms  \citep{candes2006near, candes2006robust, donoho2006compressed}. However, the additional requirement of linear decoding takes the problem out of the classical compressed sensing, into \textit{linear} compressed sensing.

    Our main theoretical result establishes nearly-matching upper and lower bounds for linear compressed sensing. We prove that $d = \Omega_\epsilon(\frac{k^2}{\log k}\log (m/k))$ is required while $d = O_\epsilon(k^2\log m)$ suffices. The lower bound establishes a quantitative gap between classical and linear compressed setting, illustrating how linear accessibility is a meaningfully stronger hypothesis than linear representation alone. The upper bound confirms that neurons can store an exponential number of features under the LRH, giving theoretical evidence for the ``superposition hypothesis'' \citep{elhage2022toy}.

    The upper bound proof uses standard random constructions of matrices with approximately orthogonal columns. The lower bound proof uses rank bounds for near-identity matrices \citep{alon2003problems} together with Tur\'an's theorem (bounding the number of edges in clique-free graphs). We also show how our results do and do not constrain the geometry of feature representations and extend our results to allow decoders with an activation function and bias.
\end{abstract}

\section{Introduction}
Towards understanding how language models work, a single principle has been remarkably useful: the Linear Representation Hypothesis (LRH) states that intermediate layers of language models store features linearly. 

The observance of linear structure dates back to early work on word embeddings \citep{mikolov2013distributed, mikolov2013linguistic, arora2018linear}. \cite{alain2016understanding} demonstrated empirically that in deeper layers of vision models, features (e.g., the presence of cats or dogs) could increasingly be extracted using linear classifiers trained on activations, suggesting that deep neural networks work to arrange features linearly. Since then, work has shown that a wide range of features can be accessed using these ``linear probes,'' such as sentiment and game state \citep{nanda2023emergent, li2025language}. Leveraging the idea that features are stored as directions in activation space, a line of work has demonstrated the ability to steer language model behavior by manipulating activations along certain directions \citep{panickssery2023steering, zou2023representation, turner2023steering, wang2025persona, chen2025persona}. Moreover, the assumption of an underlying linear structure has motivated the use of sparse autoencoders to automatically extract features from language models. Strikingly, these features appear to be highly-interpretable while capturing much of the original information \citep{cunningham2023sparse, gao2024scaling, movva2025sparse}. These linearly-represented features have formed the foundation for interpretable and causal explanations for how language models perform certain computations, like writing poems or answering basic questions \citep{lindsey2025biology}.

Given the tremendous amount of activity that revolves around the LRH, recent efforts have aimed to give a formal account of the hypothesis and to derive and test implications \citep{elhage2022toy, park2023linear, costa2025flat}. In particular, it has been hypothesized that linear representation allows for \textit{superposition}, where $d$ neurons can store $m\gg d$ features \citep{elhage2022toy}. Given the vast number of relevant concepts needed to generate text, such a capability appears to be necessary given the performance of language models. While results such as the Johnson-Lindenstrauss lemma \citep{johnson1984extensions} and from compressed sensing \citep{candes2006near, candes2006robust, donoho2006compressed} are suggestive, and empirical evidence exists in toy settings \citep{elhage2022toy}, current theoretical accounts do not clearly establish bounds on the amount of possibe superposition under the LRH.

In the present work, we aim to give a mathematical framework capable of addressing the following fundamental question (stated informally here):
\begin{quote}
    (*) Under the linear representation hypothesis, how many features can one layer of $d$ neurons store?
\end{quote}
More concretely, an intermediate layer of a language model embeds an input as a $d$-dimensional vector (the activations of the $d$ neurons). Then how many features can this vector store, and under what conditions?

The present work illustrates one reason why there is a lack of clarity surrounding this basic question: different formulations of the LRH imply different answers to the question. In particular, we show a quantitative gap in the answer to (*) depending on if the LRH states that features are \textit{represented} linearly or are also \textit{accessible} linearly. These statements have often been treated as equivalent: for example, \cite{elhage2022toy} write that
\begin{quote}
    Linear representations make features ``linearly accessible.'' A typical neural network layer is a linear function followed by a non-linearity. If a feature in the previous layer is represented linearly, a neuron in the next layer can ``select it'' and have it consistently excite or inhibit that neuron. If a feature were represented non-linearly, the model would not be able to do this in a single step.
\end{quote}
Both linear representation and accessibility play an important role in work surrounding the LRH. Linear representation implies that model behavior can be altered by modifying activation space along linear directions. On the other hand, linear accessibility implies that features can be extracted using linear probes, and as \cite{elhage2022toy} note above, accessibility allows the next layer of a model to efficiently use a feature. Furthermore, both linear representation and accessibility are implicit in the design of sparse autoencoders used in practice, which utilize both linear encoding and decoding.

One key implication of the present work is that there is a meaningful way in which linear representation does not imply linear accessibility. We now describe our mathematical framework and results, illustrating this conceptual difference, and then showing how they lead to a quantitative difference.

\subsection{Mathematical Framework}
In this section, we introduce a mathematical framework that allows us to formally study implications of the linear representation hypothesis. In particular, the framework allows us to cleanly understand the difference between (linear) representation and (linear) accessibility, while also allowing us to formally quantify what it means for superposition to occur (for neuron activations to store more features than there are neurons).

The theoretical results, stated in \Cref{sec:results}, are self-contained. However, the framework here provides the more general context establishing the relationship between our mathematical results and the emerging scientific theory of language models.

\paragraph{Activations and features.} We will assume that inputs comes from some set $L$, which we call a \textbf{language} (i.e., sequences of tokens). We let $f:L\rightarrow \RR^d$ give the \textbf{activations} of $d$ neurons in an intermediate layer. We let a \textbf{feature} be a function $z_i:L\rightarrow \RR$. For example, a feature can capture if the text is a question, if it has positive sentiment, if it is about dogs, or if it should be continued with a number. Importantly, we think of the value of a feature $z_i(\ell)$ as depending only on the input text; in other words, features can exist independently of any language model.

\paragraph{Feature accessibility.} A feature $z_i$ is \textbf{$(\epsilon, S)$-recovered} by a \textbf{probe} $g_i:\RR^d\rightarrow \RR$ if
\begin{equation}
    |g_i(f(\ell)) - z_i(\ell)| < \epsilon
\end{equation}
for all $\ell \in S.$ Intuitively, this means that the value of a feature can be accessed from a text's activations $f(\ell)$, at least up to some approximation error.

\paragraph{The Linear Representation Hypothesis.} We now formalize how the LRH places further restrictions on the above framework.

\textbf{Linear representation} for a set of features $z_1,\cdots,z_m:L\rightarrow \RR$ implies that there exists a set of \textbf{representation vectors} $a_1,\cdots,a_m\in \RR^d$, such that
\begin{equation}
    f(\ell) = \sum_{i=1}^m z_i(\ell)a_i.
\end{equation}
Equivalently, there exists $A\in \RR^{d\times m}$ such that
\begin{equation}
    f(\ell) = Az(\ell),
\end{equation}
where $z(\ell)=[z_1(\ell),z_2(\ell),\cdots,z_m(\ell)]\in \RR^m.$ We refer to $z(\ell)$ as a text $\ell$'s \textbf{feature representation}. Under linear representation, activations are a function only of an input's feature representation.

\textbf{Linear accessibility} places a restriction on the probes $g_i$ (in contrast to linear representation, which restricts $f$). Formally, under linear accessibility, we require the probe $g_i$ to satisfy $g_i(x) = \langle b_i, x\rangle$ for some \textbf{probe vector} (or \textbf{linear probe}) $b_i\in \RR^d$. Therefore, a feature $z_i$ is \textbf{$(\epsilon,S)$-linearly recovered} by a linear probe $b_i$ if
\begin{equation}
    |\langle b_i, f(\ell)\rangle - z_i(\ell)| < \epsilon
\end{equation}
for all $\ell \in S.$ Then if all features $z_1,\cdots,z_m$ can be $(\epsilon,S)$-linearly recovered, this implies that there exists $B\in \RR^{d\times m}$ such that
\begin{equation}
    \norm{B^T f(\ell) - z(\ell)}_\infty < \epsilon
\end{equation}
for all $\ell\in S.$

\paragraph{Superposition.} We can now see how this framework allows us to formally reason about (*), or about how much superposition can occur under the LRH. In particular, we are interested in understanding how many features $m$ can be stored using activations in $\RR^d$ (of $d$ neurons) under the hypotheses of linear representation and linear recovery.

First note that under the linear representation, activations of an input text are a function only of its feature representation. Therefore, we can redefine activations to be given by $f:\RR^m\rightarrow \RR^d$. Under this definition, we restate linear representation to say that
\begin{equation}
    f(z) = Az
\end{equation}
for $z\in \RR^m, A\in \RR^{d\times m}.$

Now, assuming linear representation, we can ask two questions:
\begin{enumerate}
    \item[Q1.] (General accessibility.) For a fixed $m$, how many dimensions $d$ are required for there to exist $A\in \RR^{d\times m}$ and $g:\RR^d\rightarrow \RR^m$ such that
    \begin{equation}
        \norm{g(Az) - z}_\infty < \epsilon
    \end{equation}
    for all $z\in S.$
    \item[Q2.] (Linear accessibility.) For a fixed $m$, how many dimensions $d$ are required for there to exist $A,B\in \RR^{d\times m}$ such that
    \begin{equation}
        \norm{B^TAz - z}_\infty < \epsilon
    \end{equation}
    for all $z\in S.$
\end{enumerate}
In both settings, we will generally assume $S$ to be the set of $k$-sparse vectors (i.e., having at most $k$ non-zero entries), and answer in terms of $m,k,\epsilon$. In other words, we study settings where only a small number of features tend to be ``active'' for any particular input (intuitively, this is matches our intuition about language---that any given piece of text only expresses a small fraction of possible concepts).

A key result we show is that there is a quantitative gap in the answer to (1) and (2). In other words, there is a meaningful sense in which linear representation does not imply linear accessibility. In both settings, superposition can occur: an exponential number of features can be represented and accessed. However, we will show that the amount of tolerated sparsity differs significantly.

\subsection{Results}\label{sec:results}
We now state theoretical results that resolve the questions Q1 and Q2 posed above. First, we will show how classical results in compressed sensing resolve the general accessibility setting (\Cref{thm:compressed-sensing}). Then, we will show that the additional restriction to \textit{linear} probes brings us into the different setting of \textit{linear} compressed setting. The primary contributions of the present work are nearly-matching upper and lower bounds in the linear setting (\Cref{thm:upper} and \Cref{thm:lower}).

\paragraph{Past Results: Compressed Sensing.} 
Let us recall a classical result in compressed sensing \citep{candes2006near, candes2006robust, donoho2006compressed}, which exactly deals with the setting in which features are represented linearly but accessed non-linearly, answering Q1.
\begin{theorem}[Compressed Sensing]\label{thm:compressed-sensing}
    There exists a matrix $A\in \RR^{d\times m}$ with $d = O\left(k\log \frac{m}{k}\right)$
    such that for all $k$-sparse $z\in \RR^m$,
    \begin{equation}
        g(x) := \argmin_{z'\in \RR^m: x=Az'} \norm{z'}_1
    \end{equation}
    satisfies
    \begin{equation}
        g(Az) = z.
    \end{equation}
\end{theorem}

This shows that when 
\begin{equation}
    d = O\left(k\log \frac{m}{k}\right),
\end{equation}
$k$-sparse $z$ in $\RR^m$ can be linearly embedded in $\RR^d$ in a way such that $z$ can be \textit{exactly} recovered from the embedding using $\ell_1$-minimization (basis pursuit). However, basis pursuit is non-linear. We are interested in the case when features can be extracted linearly, which we discuss next.

\paragraph{Our Results: Linear Compressed Sensing.} The requirement that the recovery function be linear takes us out of classical compressed sensing settings, which generally allows arbitrary (efficient) algorithms to perform recovery. The assumption of linear recovery (the setting of Q2) brings us into the regime of what could be called \textit{linear compressed sensing}.

Let $d(m,k,\epsilon)$ be the smallest choice of $d$ such that there exists $A,B\in \RR^{d\times m}$ such that
\begin{equation}
    \norm{B^TAz - z}_\infty < \epsilon
\end{equation}
for all $k$-sparse $z\in [-1,1]^m$. (Here, any bounded interval would suffice.)

Our main results provide nearly-matching upper and lower bounds on $d(m,k,\epsilon).$

\begin{theorem}[Upper Bound]\label{thm:upper}
    \begin{equation}
        d(m,k,\epsilon) = O_\epsilon\left(k^2\log m\right).
    \end{equation}
\end{theorem}
This upper bound demonstrates that embeddings can in fact store an exponential number of features relative to embedding dimension.

\begin{theorem}[Lower Bound]\label{thm:lower}
For $\epsilon > \frac{k^{3/2}\sqrt{5}}{\sqrt{m}},$
    \begin{equation}
        d(m,k,\epsilon) = \Omega_\epsilon\left(\frac{k^2}{\log k}\log \frac{m}{k}\right).
    \end{equation}
\end{theorem}

(Here, the restriction $\epsilon > \frac{k^{3/2}\sqrt{5}}{\sqrt{m}}$ is not very consequential, since we are generally interested in the regime where $k = O(\log m).$)

The lower bound establishes a quantitative gap between compressed sensing and linear compressed sensing. In particular, while $d$ must be essentially quadratic in $k$ in linear compressed sensing, it need only be linear in $k$ in classical compressed sensing.

\paragraph{Further results.} There has also been general interest in the ``geometry'' of features stored in language models \citep{park2023linear}. For example, it is sometimes stated that (unrelated) features are represented as \textit{orthogonal} directions under the LRH \citep{hindupur2025projecting}. In \Cref{sec:feature-geometry}, we show what our theoretical results imply (and do not imply) about feature geometry. In particular, \Cref{prop:geometry-1} shows that the lower bound in \Cref{thm:lower} can be achieved even for non-intuitive arrangements of feature directions: features can be represented and probed by approximately orthogonal directions, while different features can be represented by highly-correlated directions and also probed by highly-correlated directions. These results  highlight the distinction between representation and probe vectors (also addressing uncertainty about whether encoder and decoder weights in sparse autoencoders correspond to a feature's ``true'' direction), and how the major requirement for linear feature representation and accessibility is that a feature's representation vector must be approximately orthogonal to other features' probe vectors (not other features' representation vectors). In \Cref{prop:geometry-2}, we then show how bounds on the magnitude of feature vectors limit the amount of ``unusualness'' in feature geometry.

In \Cref{sec:binary}, we then show that the lower bound in \Cref{thm:lower} also applies to more general settings. In particular, in \Cref{thm:threshold} we show that the bound applies in a setting where features are binary and the task is to linearly classify inputs according to each feature. As a corollary, the result implies that adding a bias and non-linear activation function to a linear probe does not significantly increase the number of features a model can store.

\subsection{Intuition and Proof Techniques}
We begin by developing some geometric intuition for the problem. Let $a_1,\cdots,a_m\in \RR^d$ be the columns of $A\in \RR^{d\times m}$ and $b_1,\cdots,b_m\in \RR^d$ be the columns of $B\in \RR^{d\times m}$. For each $i\in [m]$, we recall that $a_i$ is feature $i$'s \textbf{representation vector} and $b_i$ is feature $i$'s \textbf{probe vector}. Now let $\hat{z} = B^TAz.$ We are interested in when
\begin{equation}\label{eq:salmon}
    \norm{\hat{z}-z}_\infty < \epsilon \iff |\hat{z}_i - z_i| < \epsilon,\quad \forall i\in [m].
\end{equation}
Note that
\begin{equation}
    \hat{z}_i = \langle b_i, Az\rangle = \sum_{j=1}^m z_j \langle b_i, a_j\rangle.
\end{equation}
Therefore, the value $\langle b_i, a_j\rangle$ tells us how much the probe $b_i$ for the $i$-th feature responds to presence of the $j$-th feature. If $|\langle b_i,a_j\rangle|$ is large for $i\neq j$, that means that a text that contains feature $j$ will contribute to the measurement of feature $i$. In this case, we say that feature $j$ interferes with the measurement of feature $i$. If $\langle b_i,a_j \rangle > 0$, this is positive interference (presence of feature $j$ will lead to inflated measurement of feature $i$). If $\langle b_i,a_j \rangle < 0$, this is negative interference (presence of feature $j$ will lead to deflated measurement of feature $i$).

Roughly speaking, for \eqref{eq:salmon} to hold, there must be low interference; equivalently, $B^TA$ must be small off-diagonal. A general step in our analysis is to consider how $k$-sparsity limits the amount by which interference across multiple features can compound.

\paragraph{Upper bound sketch.} We can give a fairly straightforward construction to establish the upper bound. In particular, we can give a construction where $A=B$. Then it suffices to find a matrix $A$ such that $\langle a_i, a_i\rangle = 1$ and such that $\langle a_i, a_j\rangle < \frac{\epsilon}{k}.$ Leveraging the $k$-sparsity of $z$, this would show that the total interference for any feature is bounded by $k \cdot \frac{\epsilon}{k} = \epsilon$. We can then use standard constructions of incoherent matrices (i.e., matrices with approximately-orthonormal columns), such as appropriately-scaled random Gaussian matrices, to obtain the desired bound. 

\paragraph{Initial attempt at lower bound.}
To prove the lower bound, a natural first inclination is to show that among a sufficiently large set of directions in $\RR^d$, there must exist correlated directions. In particular, it is possible to show that for $d$ small relative to $m$, there exists $i\in [m]$ and $k$ choices of $j$ such that $|\langle a_i,a_j\rangle| > \frac{\epsilon}{k}$ and such that $\langle a_i,a_j\rangle$ have the same sign (and thus compound in interference). However, this approach does not generally work when $A\neq B,$ since we do not strictly require representation directions of different features (or probe directions of different features) to be orthogonal. For $i\neq j$, it is only necessary for the inner product of feature $i$'s representation vector to be orthogonal to feature $j$'s \textit{probe} vector. This means that there may exist ``good'' sets of representation directions and probe directions among which many are pairwise correlated.

In fact, in \Cref{sec:feature-geometry}, we show that there are constructions with $B\neq A$ where probe directions are highly correlated with each other and representation directions are also highly correlated with each other. So while showing the existence of correlated vectors gives a lower bound in the case $A=B$, it does not suffice in general.

\paragraph{Lower bound sketch.} 
It would suffice to show that if $\langle b_i, a_i\rangle$ is near $1$ (as is required if $B^TAe_i \approx e_i$) and $d$ is small, then there must exist many pairs $i\neq j$ such that $|\langle b_i, a_j \rangle|$ is large. In particular, we would like to find a feature $i$ such that there are $k$ choices of $j$ such that $|\langle b_i, a_j \rangle| > \frac{\epsilon}{k}$ and such that the interference is all in the same direction. To do this, it suffices to find $2k$ such vectors $a_j$ with $|\langle b_i, a_j \rangle| > \frac{\epsilon}{k}$ (by the pigeonhole principle).

We employ a few ideas. The first is to use a result of \cite{alon2003problems}, which demonstrates that low-rank matrices that are $1$ on the diagonal must have a large off-diagonal entry. (This result also notably provides a near-optimal lower bound on the Johnson-Lindenstrauss Lemma.) In particular, we know that $B^TA$ has rank at most $d$, and we can minimally scale rows to allow diagonal entries to be $1$. Note, however, that applying the result to $B^TA$ itself does not suffice since this only guarantees the existence of a single interfering pair of features.

Therefore, the second main idea is to apply the result instead to large principal submatrices of $B^TA$, which demonstrates that any submatrix of $B^TA$ must have a large off-diagonal entry. Equivalently, among any large subset $R\subseteq [m]$, there must exist a pair $i,j\in R$ with $i\neq j$ such that $\langle b_i, a_j \rangle$ is large. 

Finally, we consider a graph reduction and apply Tur\'an's theorem. Consider a graph $G$ with vertices $[m]$ and such that there is an edge between $i$ and $j$ if and only if $|\langle b_i, a_j\rangle| > \frac{\epsilon}{k}$ or $|\langle b_j, a_i\rangle| > \frac{\epsilon}{k}$. Then $G$ has no independent set of size $r$. Then we can apply Tur\'an's theorem (which gives an upper bound on the number of edges in a graph without an $r$-clique) on $\overline{G}$, hence providing a lower bound on the number of edges in $G$, which in turns gives a lower bound on the number of pairs of interfering features. Then we can use the pigeonhole principle to show that there must exist some probe vector that has large interference with \textit{many} representation vectors, as desired.

\begin{figure}
\centering
\begin{tikzpicture}[
    >=stealth,
    scale=2.5,
    vector/.style={->, very thick},
    axis/.style={very thin, gray!60},
    grid/.style={thin, gray!20}
]

\draw[grid] (-0.1,-0.8) grid (1.2,1.2);

\draw[axis,->] (-0.1,0) -- (1.2,0) node[right] {};
\draw[axis,->] (0,-0.8) -- (0,1.3) node[above] {};

\draw[vector,Bittersweet] (0,0) -- (0.5,0.866) 
   node[above right] {$a_{\texttt{cat}} = \left(\frac{1}{2}, \frac{\sqrt{3}}{2}\right)$};

\draw[vector,Bittersweet] (0,0) -- (1,0) 
   node[below right] {$a_{\texttt{happy}} = (1,0)$};

\draw[vector,ForestGreen] (0,0) -- (0,1.1547) 
   node[above left] {$b_{\texttt{cat}} = \left(0, \frac{2}{\sqrt{3}}\right)$};

\draw[vector,ForestGreen] (0,0) -- (1,-0.577) 
   node[below right] {$b_{\texttt{happy}} = \left(1, -\frac{1}{\sqrt{3}}\right)$};

\end{tikzpicture}
\caption{Representation vectors $a_{\texttt{cat}}, a_{\texttt{happy}}$ and probe vectors $b_{\texttt{cat}}, b_{\texttt{happy}}$ that yield perfect linear recovery of the \texttt{cat} and \texttt{happy} features. Notice that while the representation vectors are not orthogonal, the probe vectors are able perfectly extract features since the \texttt{cat} probe is orthogonal to the \texttt{happy} representation and the \texttt{happy} probe is orthogonal to the \texttt{cat} representation. Indeed, observe that $b_{\texttt{cat}}(z_{\texttt{cat}}a_{\texttt{cat}} + z_{\texttt{happy}}a_{\texttt{happy}}) = z_{\texttt{cat}}\langle b_{\texttt{cat}}, a_{\texttt{cat}}\rangle + z_{\texttt{happy}}\langle b_{\texttt{cat}}, a_{\texttt{happy}}\rangle = z_{\texttt{cat}}\cdot 1 + z_{\texttt{happy}}\cdot 0.$}
\end{figure}
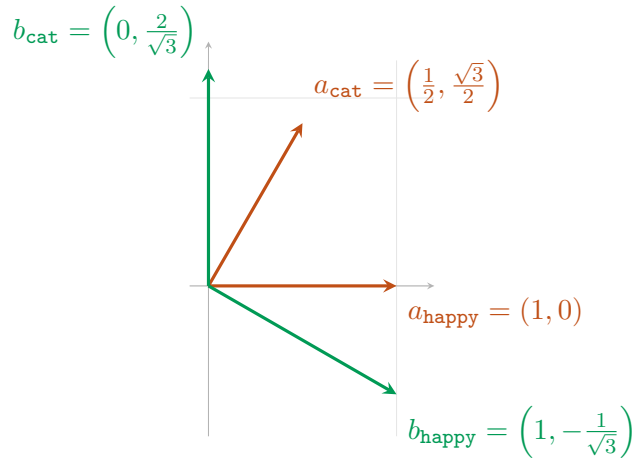

\subsection{Contributions: Implications for the Linear Representation Hypothesis and the Theory of Language Models}

While linear representations have long been a central character in the empirical study of language models \citep{mikolov2013distributed,alain2016understanding}, it has been unclear whether the LRH can provide a theoretical basis by which to study language models. Here, we explore this possibility by providing mathematical foundations for the linear representation hypothesis. The framework we introduce provides definitional clarity on the LRH, cleanly distinguishing the distinct assumptions of linear representation and linear accessibility, and allowing us to formally consider the amount of superposition (i.e., the number of features stored relative to the number of neurons) possible under the LRH.

The framework allows us to then derive theoretical implications of the linear representation hypothesis. Mathematically, we establish a quantitative gap between compressed sensing and \textit{linear} compressed sensing. Our main results demonstrate dependencies between superposition (the number of features stored relative to the number of neurons), sparsity (the number of features active per input), accessibility (the precision by which the next layer can access features stored in the previous layer), and feature geometry (the arrangement of linear directions corresponding to features). 

A primary implication of our work is that linear representations---in addition to being conceptually elegant and highly-compatible with neural network operations---can be remarkably expressive. Under reasonable regimes of sparsity, a single layer of a language model can store an exponential number of features relative to neurons, such that subsequent layers can easily access these features. This amount of expressivity suggests the possibility that the LRH can serve as a serious foundation for deeper theoretical understand of language models.

\section{Upper Bound Proof}\label{sec:upper-bound}
In this section, we prove \Cref{thm:upper}. In particular, we establish that there exists a matrix $A\in \RR^{d\times m}$ with $d = O_{\epsilon}(k^2\log m),$ such that
\begin{equation}
    \norm{A^TAz - z}_\infty < \epsilon
\end{equation}
for all $k$-sparse $z\in [-1,1]^m.$ 

Intuitively, we want to find $A$ such that $A^TA$ is close to the identity matrix---i.e., that $A$ is approximately orthonormal. The following definition will be useful.

\begin{definition}[$\mu$-incoherence]
    Say that a matrix $C$ with columns $c_1,\cdots,c_m$ is $\mu$-incoherent if 
    \begin{equation}
        \langle c_i, c_i\rangle = 1
    \end{equation}
    for all $i\in [m]$, and
    \begin{equation}
        |\langle c_i, c_j\rangle| < \mu
    \end{equation}
    for all $i\neq j$.
\end{definition}

We can then show the following (standard) result, which demonstrates that random matrices are incoherent. (Equivalently, it shows that we can select an exponential number of approximately orthogonal directions.) We give a proof in the case of an appropriately scaled matrix of independent Rademachers, but the result also holds, for example, for i.i.d. Gaussian entries.
\begin{lemma}[Incoherence of Random Matrices]\label{prop:incoherent}
    Let $C\in \RR^{d\times m}$ be a matrix with i.i.d. scaled Rademacher entries $\pm\frac{1}{\sqrt{d}}$. For
    \begin{equation}
        d = O\left(\frac{\log m + \log \frac{1}{\delta}}{\mu^2}\right),
    \end{equation}
    $C$ is $\mu$-incoherent with probability at least $1 - \delta$.
\end{lemma}
\begin{proof}
    First observe that $|\langle c_i, c_i\rangle| = \sum_{i=1}^d \frac{1}{d} = 1.$ Next, observe that
    \begin{equation}
        |\langle c_i, c_j\rangle| = \frac{1}{d}\sum_{\ell=1}^d X_\ell,
    \end{equation}
    where $X_i\in \{-1,1\}$ are i.i.d.. Then note that
    \begin{align}
        \Pr[|\langle c_i, c_j\rangle| > \mu] = \Pr\left[\left|\sum_{\ell=1}^d X_\ell\right|\ge \mu d\right] \le 2\exp\left(-\frac{d\mu^2}{2}\right),
    \end{align}
    where we apply Hoeffding's inequality. Applying a union bound, $\langle c_i, c_j \rangle < \mu$ for all $i\neq j$ with probability at least
    \begin{equation}
        1 - \binom{m}{2}2\exp\left(-\frac{d\mu^2}{2}\right) \ge 1 - m^2\exp\left(-\frac{d\mu^2}{2}\right).
    \end{equation}
    Taking $d = \frac{2}{\mu^2}(2\log m - \log \delta),$ we have $m^2\exp\left(-\frac{d\mu^2}{2}\right) \le \delta,$ as desired.
\end{proof}

In particular, the result guarantees the existence of a $\mu$-incoherent matrix with $d = O(\frac{\log m}{\mu^2}).$ 

\paragraph{Proof of \Cref{thm:upper}.}
Let $\hat{z}=A^TAz$ and suppose that $A$ is $\mu$-incoherent. Then if $z\in [-1,1]^m$ has support $T$ with $|T|\le k$, we have that
\begin{align}
    \hat{z}_i = \left\langle a_i, \sum_{j=1}^m z_ja_j\right\rangle
    &= z_i\langle a_i,a_i\rangle + \sum_{j\in T\setminus\{i\}} z_j\langle a_i,a_j\rangle\\
    &= z_i + \sum_{j\in T\setminus\{i\}} z_j\langle a_i,a_j\rangle.
\end{align}
Therefore,
\begin{equation}
    \left|\hat{z}_i - z_i\right| < \sum_{j\in T\setminus\{i\}} |z_j\langle a_i,a_j\rangle| < k\mu,
\end{equation}
where in the last inequality, we used that $A$ is $\mu$-incoherent, $|T|\le k,$ and $z_j\in [-1,1]$. This implies that $\norm{\hat{z}-z}_\infty < k\mu.$
Taking $\mu = \frac{\epsilon}{k}$ and applying \Cref{prop:incoherent}, we have
\begin{equation}
    d(m,k,\epsilon) = O\left(\frac{k^2\log m}{\epsilon^2}\right),
\end{equation}
as desired.

\begin{remark}
The construction we have provided has $B=A$, meaning that features are represented with and probed with vectors in the same direction. In fact, as we show in \Cref{sec:feature-geometry}, there exist constructions where a feature's representation and probe vectors are themselves approximately orthogonal, while all feature directions are highly correlated and all probe directions are also highly correlated.

The observation here also implies that attempts to show lower bounds by proving that sufficiently large sets of vectors must contain a pair of correlated vectors does not succeed, since collections of ``good'' representation and probe directions can in general contain many highly correlated directions.
\end{remark}

\section{Lower Bound Proof}\label{sec:lower-bound}
In this section, we prove \Cref{thm:lower}. In particular, we establish that for $\epsilon > \frac{k^{3/2}\sqrt{5}}{\sqrt{m}},$ if there exists $A,B\in \RR^{d\times m}$ such that $\norm{B^TAz-z}_\infty < \epsilon$ for all $k$-sparse $z\in [-1,1]^m,$ then $d = \Omega_\epsilon(\frac{k^2}{\log k}\log \frac{m}{k}).$

\paragraph{Preliminary observations.} Consider $A,B\in \RR^{d\times m}$ such that
$\norm{B^TAz - z}_\infty < \epsilon$
for all $k$-sparse $z\in [-1,1]^m.$ Taking $z=e_i,$ this implies that
\begin{equation}\label{eq:leek}
    |(B^TAe_i)_i - 1| < \epsilon \implies \langle b_i, a_i\rangle \in 1\pm \epsilon.
\end{equation}
for all $i\in [m]$. Furthermore, taking $z=\sum_{j\in T}e_j$ where $T\subseteq [m]\setminus\{i\}, |T| \le k$ implies that
\begin{equation}\label{eq:fennel}
    |(B^TAz)_i - 0| < \epsilon \implies |\langle b_i, \sum_{j\in T}a_j\rangle| < \epsilon.
\end{equation}
To prove \Cref{thm:lower}, our approach will be to show that if $A,B\in \RR^{d\times m}$ satisfy \eqref{eq:leek} with $d < O_\epsilon(\frac{k^2}{\log k} \log \frac{m}{k})$, then we can find a contradiction to \eqref{eq:fennel}.

\paragraph{Finding many features with correlated interference.}
It would suffice to find $i\in [m]$ and a set $T^*\subseteq [m]\setminus \{i\}$ such that $|T^*|\ge 2k$ and
\begin{equation}
    |\langle b_i, a_j\rangle| > \frac{\epsilon}{k}
\end{equation}
for all $j\in T.$ This would imply (by the pigeonhole principle) the existence of a subset $T\subseteq T^*$ with $|T|=k$ such that $|\langle b_i, a_j\rangle| > \frac{\epsilon}{k}$
for all $j\in T$ and such that $|\langle b_i, a_j\rangle|$ has the same sign for all $j\in T.$ In this case,
\begin{equation}
    |\langle b_i, \sum_{j\in T}a_j\rangle| = \sum_{j\in T} |\langle b_i, a_j\rangle| > k\cdot\frac{\epsilon}{k} = \epsilon,
\end{equation}
providing the needed contradiction. The remainder of the proof consists of showing that when $d < O_\epsilon(\frac{k^2}{\log k} \log \frac{m}{k})$, there must exist such $i, T^*$.

\paragraph{Bounding the rank of matrices with small off-diagonals.}
Note that the matrix $B^TA$ has entries $(B^TA)_{ij} = \langle b_i, a_j\rangle$. Also, $\rank(B^TA)\le d$. Our aim is to show that when $d$ is small relative to $m$, under the assumption of \eqref{eq:leek}, there must exist a row of $B^TA$ that has $2k$ off-diagonal entries with absolute value at least $\frac{\epsilon}{k}$. At a high level, the idea we pursue is that matrices that are near-zero off-diagonal (and large on the diagonal) must have large rank.

We begin by stating a result given by \cite{alon2003problems}. The result shows that any square matrix that is one on the diagonal and bounded off-diagonal must have large rank. Of course, if the off-diagonal is 0, the matrix is the identity and thus has full rank. The result shows that we can bound the rank from below given small perturbations.
\begin{lemma}[\cite{alon2003problems}, Theorem 9.3]\label{lem:alon}
Suppose $\frac{1}{\sqrt{n}} < \epsilon < \frac{1}{2}$. Then let $D\in \RR^{n\times n}$ such that $D_{ii}=1$ for all $i\in [n]$ and $|D_{ij}| < \epsilon$ for all $i\neq j$. Then
\begin{equation}
    \rank(D) \ge \Omega\left(\frac{1}{\epsilon^2\log(\frac{1}{\epsilon})}\log n\right).
\end{equation}
\end{lemma}

It is useful to state a corollary, which generalizes the result to the setting where diagonal entries are bounded away from 0.
\begin{corollary}\label{lem:rank}
Let $\gamma \in (0,1]$ and $\frac{\gamma}{\sqrt{n}} < \epsilon < \frac{\gamma}{2}$. Then let $C\in \RR^{n\times n}$ such that $C_{ii}\ge \gamma$ for all $i\in [n]$ and $|C_{ij}| < \epsilon$ for all $i\neq j$. Then
\begin{equation}
    \rank(C) \ge \Omega\left(\frac{\gamma^2}{\epsilon^2\log(\frac{\gamma}{\epsilon})}\log n\right).
\end{equation}
\end{corollary}

\begin{proof}
    Let $D$ be the matrix obtained by scaling each entry in the $i$-th row of $C$ by $\frac{1}{C_{ii}}$. Then $\rank(C)=\rank(D)$. Then $D_{ii}=1$ for all $i\in [m]$ and $|D_{ij}| = \frac{|C_{ij}|}{C_{ii}}\le \frac{\epsilon}{\gamma}$ for all $i\neq j$. Then the result follows by applying \Cref{lem:alon} with $\epsilon = \frac{\epsilon}{\gamma}.$
\end{proof}

Now notice that $|(B^TA)_{ii}| > 1 - \epsilon$ for all $i\in [m]$ (by \eqref{eq:leek}). Therefore, we can apply \Cref{lem:rank} to $B^TA$, taking $\gamma = 1 - \epsilon.$ However, applying \Cref{lem:rank} directly like this does not provide the result we desire, since it only shows the existence of a single large off-diagonal entry in $B^TA$. (A sufficiently large such entry would also contradict \eqref{eq:fennel}, but applying \Cref{lem:rank} in this way would only imply a $\Omega_\epsilon (\log m)$ lower bound, losing dependence on $k$.)

\paragraph{Applying \Cref{lem:rank} to submatrices.} Instead, our approach is to apply \Cref{lem:rank} to all principal submatrices of $C=B^TA$ of a certain size. For a set $R\subseteq [m]$, let $C_R \in \RR^{|R|\times |R|}$ denote the submatrix given by taking the rows and columns of $C$ in $R$. Then $C_R$ also has diagonal entries at least $\gamma = 1 - \epsilon,$ and $\rank(C_R)\le \rank(C) \le d.$

Setting $r = \frac{m}{4k+1},$ \Cref{lem:rank} implies that there exists a constant $\alpha > 0$ such that if $|R|=r$ and
\begin{equation}
    \rank(C_R) < \alpha \frac{(1 - \epsilon)^2\log r}{(\frac{\epsilon}{k})^2\log (\frac{k(1 - \epsilon)}{\epsilon})},
\end{equation}
then $C_R$ has an off-diagonal entry with absolute value at least $\frac{\epsilon}{k}.$ (Also note that $\epsilon > \frac{k^{3/2}\sqrt{5}}{\sqrt{m}}, \implies \frac{\epsilon}{k} > \frac{1}{\sqrt{\frac{m}{4k+1}}} = \frac{1}{\sqrt{r}}$ for $k\ge 1$, satisfying the assumption in \Cref{lem:rank}.) Equivalently, for any set $|R|\subseteq [m]$, there exist $i,j\in R$ with $i\neq j$ such that $|C_{ij}| > \frac{\epsilon}{k}.$ In other words, among any $r$ features, there must exist an interfering pair.

\paragraph{Graph reduction and Tur\'an's theorem.} We would now like to show that the presence of a large off-diagonal entry in any submatrix implies the existence of a row with many large off-diagonal entries. Consider the undirected graph $G = ([m],E)$ such that $(i,j)\in E$ for $i\neq j$ if and only if $|\langle b_i, a_j \rangle| > \frac{\epsilon}{k}$ or $|\langle b_j, a_i \rangle| > \frac{\epsilon}{k}$. Then we know from above that if
\begin{equation}
    d\le \alpha \frac{(1 - \epsilon)^2\log r}{(\frac{\epsilon}{k})^2\log (\frac{k(1 - \epsilon)}{\epsilon})},
\end{equation}
then for every $R\subseteq [m]$ with $|R|=r$, there exists $i,j\in R$ with $i\neq j$ such that $(i,j)\in E.$ In other words, $G$ has no independent set of size $r$, or equivalently, $\overline{G}$ has no clique of size $r.$ Recall that Tur\'an's theorem shows that a graph with $m$ edges without an $r$-clique has at most $(1 - \frac{1}{r-1})\frac{m^2}{2} < (1 - \frac{1}{r})\frac{m^2}{2}$ edges. This implies that $G$ has at least
\begin{equation}
    \binom{m}{2} - \left(1 - \frac{1}{r}\right)\frac{m^2}{2} = \frac{m^2}{2r}-\frac{m}{2}
\end{equation}
edges. Now note that every edge in $G$ implies at least one off-diagonal entry in $B^TA$ with absolute value at least $\frac{\epsilon}{k}.$ This implies the existence of a row with at least $\frac{m}{2r} - \frac{1}{2}$ off-diagonal entries with absolute value at least $\frac{\epsilon}{k}$ (by the pigeonhole principle). Setting $r = \frac{m}{4k+1}$ implies a row in $B^TA$ with $2k$ entries with absolute values at least $\frac{\epsilon}{k}$, providing the desired contradiction.

Therefore, we must have that
\begin{equation}
    d > \alpha\frac{\log \frac{m}{4k+1}}{(\frac{\epsilon}{(1-\epsilon)k})^2\log \frac{(1-\epsilon)k}{\epsilon}} = \Omega_\epsilon\left(\frac{k^2}{\log k}\log \frac{m}{k}\right),
\end{equation}
as desired.

\section{The geometry of feature directions}\label{sec:feature-geometry}
One benefit of the linear representation hypothesis is that it provides a geometric conception of how language models represent concepts. In particular, features are thought of as directions in activation space. This allows us to state geometric relationships between different features. For example, it is often suggested that unrelated features are represented by \textit{orthogonal} directions. In this section, we establish what is implied and not implied about the geometry of feature vectors under our theoretical framework.

\Cref{prop:geometry-1} demonstrates two unintuitive facts. First, a feature can be represented and accessed by \textit{different} vectors; in fact, a feature's representation and probe directions can be almost perfectly orthogonal. Second, different features can have representation vectors that are almost perfectly in the same direction; likewise, different features can have probe vectors that are almost perfectly in the same direction.

\Cref{prop:geometry-2} then establishes that this unusual behavior can be eliminated if we further restrict the norms of both the representation and probe vectors. In particular, if we bound the norm of probe and representation vectors to be near $1$, probe and representation directions must be closely aligned, and different features must represented and probed by approximately orthogonal directions.

\begin{proposition}\label{prop:geometry-1}
    Consider constants $\epsilon, \delta\in (0,1)$. Then for $d = O_{\epsilon,\delta}(k^2\log m),$ there exists $A,B\in \RR^{d\times m}$ such that
    \begin{equation}
        \norm{B^TAz - z}_\infty < \epsilon
    \end{equation}
    for all $k$-sparse $z\in [-1,1]^m,$ and where
    \begin{enumerate}
        \item[(i)] $|\langle \frac{a_i}{\norm{a_i}}, \frac{b_i}{\norm{b_i}}\rangle| < \delta + o(1)$ for all $i\in [m]$,
        \item[(ii)] $\langle \frac{a_i}{\norm{a_i}}, \frac{a_j}{\norm{a_j}}\rangle
        > 1 - \delta - o(1)$ for all $i\neq j$,
        \item[(iii)] $\langle \frac{b_i}{\norm{b_i}}, \frac{b_j}{\norm{b_j}}\rangle > 1 - \delta - o(1)$ for all $i\neq j$.
    \end{enumerate}
\end{proposition}

\begin{proof}
    We give a construction. Let $c_1,c_2,\cdots,c_m,a^*,b^*$ be the columns of a $\mu$-incoherent matrix $\RR^{d\times m+2}$. Then set
    \begin{align}
        a_i &= c_i + \lambda a^*\\
        b_i &= c_i + \lambda b^*
    \end{align}
    for $\lambda = \sqrt{\frac{1}{\delta} - 1}.$
    We begin by showing that for an appropriately chosen $\mu$, $\norm{B^TAz - z}_\infty < \epsilon$ for all $k$-sparse $z\in [-1,1]^m.$ Now let $\hat{z} = B^TAz$ for any $z\in [-1,1]^m$ with support $T$. Then
    \begin{equation}
        \hat{z}_i = \left\langle b_i, \sum_{j=1}^m z_ja_j \right\rangle
        = z_i\langle b_i, a_i\rangle + \sum_{j\in T\setminus \{i\}} z_j\langle b_i, a_j\rangle.
    \end{equation}
    Therefore,
    \begin{align}
        |\hat{z}_i - z_i| &\le |z_i(\langle b_i, a_i\rangle - 1)| + \sum_{j\in T\setminus \{i\}} |z_j\langle b_i, a_j\rangle|\\
        &\le |z_i(\langle b_i, a_i \rangle - 1)| + \sum_{j\in T\setminus \{i\}} |\langle b_i, a_j\rangle|,
    \end{align}
    where the last inequality follows from observing that $z_j \in [-1,1]$ for all $j\in [m].$ Now observe that
    \begin{align}
        |z_i(\langle b_i, a_i\rangle - 1)| &= |z_i||\langle c_i + \lambda b^*, c_i + \lambda a^*\rangle - 1|\\
        &\le |z_i|\left(|\langle c_i, \lambda a^*\rangle| + |\langle \lambda b^*, c_i \rangle| + |\langle \lambda b^*, \lambda a^* \rangle|\right) \label{eq:apple-1}\\
        &= |z_i|(2\lambda  + \lambda^2)\mu\label{eq:apple-2}\\
        &< |z_i|(1 + \lambda)^2\mu,
    \end{align}
    where \eqref{eq:apple-1} follows because $|\langle c_i, c_i \rangle| = 1$ and \eqref{eq:apple-2} follows from $\mu$-incoherence. Similarly,
    \begin{align}
        |\langle b_i, a_j\rangle| &= |\langle c_i + \lambda b^*, c_i + \lambda a^*\rangle|\\
        &= |\langle c_i, c_j\rangle| + |\langle c_i, \lambda a^*\rangle| + |\langle \lambda b^*, c_i \rangle| + |\langle \lambda b^*, \lambda a^* \rangle|\\
        &\le (1+2\lambda+\lambda^2)\mu = (1 + \lambda)^2\mu.
    \end{align}
    This implies that 
    \begin{equation}
        |z_i(\langle b_i, a_i \rangle - 1)| + \sum_{j\in T\setminus \{i\}} |\langle b_i, a_j\rangle| < k (1 + \lambda)^2\mu,
    \end{equation}
    where we use the fact that if $|z_i| > 0$, then $|T\setminus \{i\}|\le k-1.$ Therefore, it suffices to take $\mu = \frac{\epsilon}{k(1+\lambda^2)}$ to ensure that $\norm{\hat{z}-z}_\infty < \epsilon.$ From \Cref{prop:incoherent}, there exists a $\mu$-incoherent matrix in $\RR^{d\times m+2}$ for
    \begin{equation}
        d = O\left(\frac{k^2(1+\lambda)^4}{\epsilon^2}\log (m+2)\right) = O\left(\frac{k^2(1+\lambda)^4}{\epsilon^2}\log m\right) = O_{\epsilon,\delta}(k^2\log m).
    \end{equation}

    \paragraph{Establishing (i).} We now show that the construction satisfies (i). Observe that for all $i\in [m],$
    \begin{align}
        \left|\left\langle \frac{a_i}{\norm{a_i}}, \frac{b_i}{\norm{b_i}} \right\rangle\right|
        &= \frac{|\langle c_i + \lambda a^*, c_i + \lambda b^*\rangle|}{\norm{a_i}\norm{b_i}}\\
        &\le \frac{|\langle c_i, c_i\rangle| + |\langle c_i, \lambda b^*\rangle| + |\langle \lambda a^*, c_i \rangle| + |\langle \lambda a^*, \lambda b^* \rangle|}{\norm{a_i}\norm{b_i}}\\
        &\le \frac{1 + (2\lambda  + \lambda^2)\mu}{\norm{a_i}\norm{b_i}} = \frac{1 + \frac{(2\lambda  + \lambda^2)\epsilon}{k(1+\lambda^2)}}{\norm{a_i}\norm{b_i}} = \frac{1 + o(1)}{\norm{a_i}\norm{b_i}}.\label{eq:dog}
    \end{align}
    Note that
    \begin{align}
        \langle a_i, a_i\rangle &= \langle c_i + \lambda a^*, c_i + \lambda a^* \rangle\\ &= \langle c_i, c_i\rangle + 2\langle c_i, \lambda a^*\rangle + \langle \lambda a^*, \lambda a^*\rangle\\
        &\le 1 + \lambda^2 - 2\lambda\mu =  1 + \lambda^2 - \frac{2\lambda\epsilon}{k(1+\lambda^2)} = 1 + \lambda^2 - o(1).
    \end{align}
    Therefore, $\norm{a_i} \le \sqrt{1 + \lambda^2 - o(1)}$. Analogously, $\norm{b_i} \le \sqrt{1 + \lambda^2 - o(1)}.$
    So we have, plugging into \eqref{eq:dog},
    \begin{align}
         \left\langle \frac{a_i}{\norm{a_i}}, \frac{b_i}{\norm{b_i}} \right\rangle
         \le \frac{1 + o(1)}{1 + \lambda^2 - o(1)} = \frac{1}{1 + \lambda^2} + o(1) = \delta + o(1),
    \end{align}
    where we recall that we set $\lambda = \sqrt{\frac{1}{\delta} - 1}.$

    \paragraph{Establishing (ii,iii).} We now show that the construction satisfies (ii). (iii) follows analogously. For all $i\neq j$,
     \begin{align}
        \left|\left\langle \frac{a_i}{\norm{a_i}}, \frac{a_j}{\norm{a_j}} \right\rangle\right|
        &= \frac{|\langle c_i + \lambda a^*, c_j + \lambda a^*\rangle|}{\norm{a_i}\norm{a_j}}\\
        &= \frac{\langle c_i, c_j\rangle + \langle c_i, \lambda a^*\rangle + \langle \lambda a^*, c_j \rangle + \langle \lambda a^*, \lambda a^* \rangle}{\norm{a_i}\norm{a_j}}\\
        &\ge \frac{\lambda^2 - (1 + 2\lambda)\mu}{1 + \lambda^2 + o(1)}\\
        &= \frac{\lambda^2 - o(1)}{1 + \lambda^2 + o(1)} = 1 - \frac{1}{1 + \lambda^2} - o(1) = 1 - \delta - o(1),
    \end{align}
    where we again recall that where we recall that $\lambda = \sqrt{\frac{1}{\delta} - 1}.$
    
\end{proof}

\begin{remark}
    In practice, different representation and probe vectors are seen in different decoder and encoder weights in SAEs, respectively (e.g., \citep{templeton2024scaling, paulo2025sparse}). This had led to some confusion about which sets of weights are the ``true'' feature vectors. \Cref{prop:geometry-1} illustrates that under the LRH, features have both a representation vector and a probe direction, and that these directions can be \textit{distinct}.
\end{remark}

We now show how the ``unusual'' behavior in \Cref{prop:geometry-1} is controlled by the norm of the representation and probe vectors.

\begin{proposition}\label{prop:geometry-2}
    Consider any constants $\epsilon > 0, \gamma \ge 1$. Then if $\norm{a_i},\norm{b_i} \le \gamma$ for all $i\in [m]$ and
    \begin{equation}
        \norm{B^TAz - z}_\infty < \epsilon
    \end{equation}
    for all $k$-sparse $z\in [-1,1]^m$, it holds that
    \begin{enumerate}
        \item[(i)] $\langle \frac{a_i}{\norm{a_i}}, \frac{b_i}{\norm{b_i}}\rangle \ge \frac{1-\epsilon}{\gamma^2}$ for all $i\in [m],$
        \item[(ii)] $\langle \frac{a_i}{\norm{a_i}}, \frac{a_j}{\norm{a_j}}\rangle \le \frac{\epsilon\gamma^2}{1 - \epsilon} + \sqrt{1 - \frac{(1 - \epsilon)^2}{\gamma^4}}$ for all $i\neq j,$
        \item[(iii)] $\langle \frac{b_i}{\norm{b_i}}, \frac{b_j}{\norm{b_j}}\rangle \le \frac{\epsilon\gamma^2}{1 - \epsilon} + \sqrt{1 - \frac{(1 - \epsilon)^2}{\gamma^4}}$ for all $i\neq j.$
    \end{enumerate}
\end{proposition}

\begin{proof}
    We must have that $\langle a_i, b_i\rangle \ge 1 - \epsilon$ since $\norm{B^TAe_i - e_i}_\infty \le \epsilon$. It follows directly that
    \begin{equation}
       \langle a_i, b_i \rangle > 1 - \epsilon \implies \left\langle\frac{a_i}{\norm{a_i}}, \frac{b_i}{\norm{b_i}}\right\rangle \ge \frac{1-\epsilon}{\gamma^2},
    \end{equation}
    establishing (i).

    We now show (ii). By Cauchy-Schwarz, $\norm{a_i}\norm{b_i} \ge \langle a_i, b_i\rangle > 1 - \epsilon$ for all $i\in [m]$. So since $\norm{a_i},\norm{b_i} \le \lambda$ for all $i\in [m]$, we have that $\norm{a_i},\norm{b_i}\ge \frac{1-\epsilon}{\gamma^2}$ for all $i\in [m]$. Therefore, we have that
    \begin{equation}
        \left|\left\langle \frac{b_i}{\norm{b_i}}, \frac{a_j}{\norm{a_j}}\right\rangle\right| \le |\langle b_i, a_j\rangle|\frac{\gamma^2}{1-\epsilon} < \frac{\epsilon\gamma^2}{1 - \epsilon}.
    \end{equation}
    Let $\Tilde{a_i}$ denote $\frac{a_i}{\norm{a_i}}$ and $\Tilde{b_i}$ denote $\frac{b_i}{\norm{b_i}}$ for all $i\in [m].$ Then write
    \begin{align}
        \Tilde{a_i} &= \alpha_i\Tilde{b_i} + \Tilde{a_i}_\perp\\
        \Tilde{a_j} &= \alpha_j\Tilde{b_i} + \Tilde{a_j}_\perp,
    \end{align}
    where $\alpha_i = \langle \Tilde{a_i}, \Tilde{b_i}\rangle$,$\alpha_j = \langle \Tilde{a_j}, \Tilde{b_j}\rangle$, and $\Tilde{a_i}_\perp, \Tilde{a_j}_\perp \perp \Tilde{b_i}.$ Then
    \begin{equation}
        \langle \Tilde{a_i}, \Tilde{a_j}\rangle = \alpha_i\alpha_j + \langle \Tilde{a_i}_\perp, \Tilde{a_j}_\perp\rangle.
    \end{equation}
    By Cauchy-Schwarz,
    \begin{equation}
        |\langle\Tilde{a_i}_\perp, \Tilde{a_j}_\perp\rangle| \le \norm{\Tilde{a_i}_\perp}\norm{\Tilde{a_j}_\perp} = \sqrt{(1 - \alpha_i^2)(1 - \alpha_j^2)}.
    \end{equation}
    Therefore,
    \begin{equation}
        |\langle \Tilde{a_i}, \Tilde{a_j}\rangle| \le |\alpha_i||\alpha_j| + \sqrt{(1 - \alpha_i^2)(1 - \alpha_j^2)}.
    \end{equation}
    We have that $1\ge \alpha_i > \frac{1 - \epsilon}{\gamma^2}$ and $0 < \alpha_j < \frac{\epsilon\gamma^2}{1 - \epsilon}.$ Therefore,
    \begin{equation}
        |\langle \Tilde{a_i}, \Tilde{a_j}\rangle| \le \frac{\epsilon\gamma^2}{1 - \epsilon} + \sqrt{1 - \frac{(1 - \epsilon)^2}{\gamma^4}},
    \end{equation}
    as desired.
    (iii) follows analogously.
\end{proof}

In particular, taking $\epsilon\rightarrow 0, \gamma = 1$ in \Cref{prop:geometry-2}, we get that
\begin{enumerate}
    \item[(i)] $\langle \frac{a_i}{\norm{a_i}}, \frac{b_i}{\norm{b_i}}\rangle \rightarrow 1$ for all $i\in [m],$
    \item[(ii)] $\langle \frac{a_i}{\norm{a_i}}, \frac{a_j}{\norm{a_j}}\rangle \rightarrow 0$ for all $i\neq j,$
    \item[(iii)] $\langle \frac{b_i}{\norm{b_i}}, \frac{b_j}{\norm{b_j}}\rangle \rightarrow 0$ for all $i\neq j.$
\end{enumerate}
This shows that when representation and probe vectors are restricted to have unit norm, the clean behavior of equal representation and probe directions emerges (meaning that each feature is given by a single direction), and that different features have orthogonal directions.

\section{Activation functions and binary classification}\label{sec:binary}
In this section, we establish a generalization of our lower bound, which simultaneous resolves two additional questions: (1) What happens when we add an activation function or bias? (2) What about when features only take on binary values (as may be the case for many common features)?

In particular, we are curious how bounds may change if we allow $g(x)=\ReLU(Wx+b)$, since individual layers in a neural network consist of a linear transformation plus a non-linear activation. We are also interested in the setting in which $z\in \{0,1\}^m$ (i.e., features are binary).

Here, we prove that when feature representations are restricted to the Boolean cube $\{0,1\}^m,$ $\Omega(k^2\log (m/k))$ embedding dimensions are required to linearly separate embeddings with and without a feature. A corollary is that adding a bias and monotonic activation function does not increase the asymptotic capacity of embeddings.

\begin{theorem}\label{thm:threshold}
    Suppose that there exist $A,B \in \RR^{d\times m}$ and $t\in \RR^m$ such that for all $i\in [m]$ and $k$-sparse $z\in \{0,1\}^m$,
    \begin{equation}\label{eq:pineapple}
        (B^TAz)_i > t_i
    \end{equation}
    if and only if $z_i=1.$ Then for $k < \sqrt{m}$,
    \begin{equation}
        d = \Omega\left(\frac{k^2}{\log k}\log\frac{m}{k}\right).
    \end{equation}
\end{theorem}
It's worth noting here that taking $\epsilon < \frac{1}{2}$ in \Cref{thm:upper} provides a nearly-matching upper bound of $\Omega(k^2\log m)$, since $\epsilon < \frac{1}{2}$ implies that $(B^TAz)_i < \frac{1}{2}$ when $z_i = 0$ and $(B^TAz)_i > \frac{1}{2}$ when $z_i = 1$. Similarly, \Cref{thm:threshold} implies \Cref{thm:lower} in the case $\epsilon < \frac{1}{2}$. However, the lower bound in \Cref{thm:lower} does not directly imply \Cref{thm:threshold} since we do not require that $\norm{B^TAz - z}_\infty$ is small here, but rather that $Az$ linearly separates positive and negative examples for each feature. That being said, the proof is very similar to that of \Cref{thm:lower}, relying on the same overall strategy.

\begin{proof}
    Suppose that there existed $A,B\in \RR^{d\times m}$ and $t\in \RR^m$ such that for all $i\in [m]$ and $k$-sparse $z\in \{0,1\}^m,$ $(B^TAz)_i > t_i$ if and only if $z_i = 1.$ Then notice that we could scale the $i$-th column of $B$ and $\gamma_i$ by any constant positive factor to get another solution. In particular, we can obtain a solution where $B^TA$ is $1$ on the diagonal. Therefore, it suffices to show that for $d \le O(\frac{k^2}{\log k}\log \frac{m}{k}),$ there does not exist $B,A,t$ satisfying \eqref{eq:pineapple} and such that $(B^TA)_{ii} = 1$ for all $i\in [m].$

    Then we have that $t_i \le 1,$ since $(B^TAe_i) > t_i$ Then, along the lines of the proof of \Cref{thm:lower}, it suffices to find $i\in [m]$ and $T^*\subseteq [m]\setminus \{i\}$ with $|T^*|=2k$ and
    \begin{equation}
        |\langle b_i, a_j \rangle| > \frac{1}{k}
    \end{equation}
    for all $j\in T^*$. Indeed, this would imply the existence of a pair $i\in [m], T\subseteq [m]\setminus\{i\}$ (with $|T|=k$), such that
    \begin{equation}
        (B^TAz)_i > k\cdot \frac{1}{k} = 1 \ge t_i,
    \end{equation}
    giving a contradiction.

    We will again apply \Cref{lem:alon} to submatrices of $C$ of size $r$. Indeed, \Cref{lem:alon} implies that there exists a constant $\alpha > 0$ such that if $|R|=r$ and
    \begin{equation}
        \rank(C_R) < \alpha \frac{\log r}{(\frac{1}{k})^2\log k} = \alpha \frac{k^2}{\log k}\log r,
    \end{equation}
    then $C_R$ has an off-diagonal entry with absolute value at least $\frac{1}{k}.$ Then we can proceed as in the proof of \Cref{thm:lower}, taking $r = \frac{m}{4k+1}$ and applying Tur\'an's theorem.
\end{proof}

The following corollary makes the simple observation that if a feature cannot be linearly separated, then it also cannot be separated by adding a bias and monotonic activation function.

\begin{corollary}\label{cor:activation-bias}
    Suppose that there exist $A,W \in \RR^{d\times m}$, $b\in \RR^m$, and a monotonically increasing activation function $\sigma:\RR\rightarrow \RR$ such that for all $i\in [m]$, $k$-sparse $z\in \{0,1\}^m$, and $x=Az,$
    \begin{equation}\label{eq:coconut}
        \sigma(W^Tx + b) > 0
    \end{equation}
    if and only if $z_i=1.$ (Here, $\sigma$ is applied component-wise.) Then for $k < \sqrt{m}$,
    \begin{equation}
        d = \Omega\left(\frac{k^2}{\log k}\log\frac{m}{k}\right).
    \end{equation}
\end{corollary}
\begin{proof}
    Consider any $i\in [m]$ and $z\in \{0,1\}^m$ such that $z_i = 1$ and $z'$ such that $z'_i = 0.$ Then $\sigma(W^TAz + b) > \sigma(W^TAz' + b).$ This implies that $(W^TAz)_i > (W^TAz')_i.$ Therefore, there exists a threshold $t_i$ such that for all $z\in \{0,1\}^m$, $(W^TAz)_i > t_i$ if and only if $z_i = 1$. The result follows by applying \Cref{thm:threshold}.
\end{proof}

The proof establishes that probes of the form $g_i(x) = \ReLU(W^Tx + b)$ implies the existence of a linear probe. In particular, this further establishes that for $\epsilon < \frac{1}{2}$, the bound in \Cref{thm:lower} cannot be lowered asymptotically by expanding the class of probes to allow a bias and activation function.

\section{Discussion}

The linear representation hypothesis is a common intuition arising from the empirical study of language models. However, in part due to a lack of formalization and definitional clarity, it has played less of a role in the theoretical study of language models. In this direction, we introduced a mathematical framework in which we can cleanly state two distinct pieces of the linear representation hypothesis. First, that in intermediate layers of language models, features are \textit{represented} linearly. Second, that these features can then be \textit{accessed} linearly. We can then formally establish the number of features $d$ neurons can linearly represent such that they can be accessed linearly. We show a quantitative gap between what can be achieved with linear and non-linear decoding. Importantly, even when assuming both linear representation and linear accessibility (moving us beyond the classical compressed sensing setting), the LRH allows for an exponential number of features to be stored relative to the number of neurons. This suggests that, while conceptually simple, the LRH may be expressive enough to explain rich and complex behavior in language models.

Our work suggests a number of interesting further directions. First, while we assume linear representation and then establish results for non-linear and linear accessibility, it is interesting to consider when linear accessibility is possible under non-linear representation. Indeed, this would help establish if previous layers can arrange features in more subtle ways while still allowing efficient access by later layers. Similarly, analysis of more sophisticated decoders (like two-layer neural networks) would be interesting. More generally, other works have proposed alternate representation hypotheses \citep{engels2024not, csordas2024recurrent, kantamneni2025language}. Formalizing these hypotheses and understanding their representational capacity could establish further insight into the plausibility of these hypotheses in comparison to the LRH. Furthermore, formalization of how multiple layers can leverage linear representations would be interesting. In particular, while we show that an exponential number of features can be represented and accessed linearly, the next layer of a neural network (say with $d_2$ neurons) consists only of $d_2$ linear probes, suggesting that each additional layer does more than extracting individual features. While the present work has established how a single layer can linearly represent many features in a way that a subsequent layer can easily retrieve, there remains much to be understood in terms of how this capability can enable multiple layers to work in tandem to achieve the remarkable phenomena we have observed.

Overall, our work lays out mathematical foundations for the linear representation hypothesis as a basis for the theoretical study of language models.

\paragraph{Acknowledgments.} We thank Sophie Greenwood, Rajiv Movva, Kevin Ren, and Divya Shanmugam for helpful discussion and feedback.

{
\bibliographystyle{alpha}
\bibliography{bib}
}

\end{document}